\documentclass[fontsize=12pt,paper=letter,bibliography=totoc]{scrartcl}
\usepackage[hmargin=1in, vmargin=1in]{geometry}

\usepackage{graphicx}
\usepackage{verbatim}
\usepackage[sumlimits,]{amsmath}
\usepackage{amsfonts}
\usepackage{amsthm}
\usepackage{amssymb}
\usepackage{mathrsfs}
\usepackage{srcltx}
\usepackage{bbm}

\usepackage{graphicx,epsfig}
\usepackage{enumitem} 
\usepackage{hyperref}
\usepackage{booktabs} 
\usepackage{caption}  

\newcommand{\calI}{\mathcal{I}}

\newcommand{\calL}{\mathcal{L}}

\newcommand{\calS}{\mathcal{S}}

\newcommand{\calN}{\mathcal{N}}
\newcommand{\calT}{\mathcal{T}}

\newcommand{\dr}{{\,\mathrm{dr}}}

\newcommand{\dOmega}{{\,\mathrm{d}\Omega}}
\newcommand{\iid}{{\mbox{i.i.d.}\,}}

\newcommand{\trace}{\mathrm{Trace}}

\newcommand{\eps}{\varepsilon}
\newcommand{\indep}{\perp\!\!\!\perp}

\newcommand{\Cov}[1]{\mbox{Cov}\left(#1\right)}
\newcommand{\Exp}[1]{\E\left\{ #1 \right\}}
\newcommand{\Exparg}[2]{\E_{#1}\left\{ #2 \right\}}

\newcommand{\Var}[1]{\mbox{Var}\left\{ #1 \right\}}

\newcommand{\E}{{\mathbb{E}}}

\newcommand{\p}{{\partial}}

\newcommand{\qbox}[1]{\quad\mbox{#1}\quad}
\DeclareMathOperator{\Real}{Re} 

\newcommand{\dsphere}{{\mathbb{S}^{d-1}}}

\newcommand{\Nat}{\mathbb{N}}

\newcommand{\Rd}{{{\mathbb{R}}^d}}

\newcommand{\D}{{\mathcal{D}^2}}
\newcommand{\ES}{{\mathcal{ES}}}
\newcommand{\diff}{{\mathrm{d}}}
\newcommand{\dW}{{\mathrm{d}W}}
\newcommand{\domega}{{\,\mathrm{d}\omega}}
\newcommand{\Htilde}{{\tilde{H}}}
\newcommand{\dponesphere}{{\mathbb{S}^{d}}}
\newcommand{\Vol}{{\mathrm{Vol}}}
\newcommand{\Sbar}{{\bar{S}}}
\newcommand{\Wbar}{{\bar{W}}}
\newcommand{\Xbar}{{\bar{X}}}
\newcommand{\Ybar}{{\bar{Y}}}

\newtheorem{theorem}{Theorem}[section]
\newtheorem*{theorem*}{Theorem}
\newtheorem{proposition}[theorem]{Proposition}
\newtheorem{lemma}[theorem]{Lemma} 
\newtheorem{corollary}[theorem]{Corollary}

\theoremstyle{definition}
\newtheorem*{definition*}{Definition}

\newtheorem{assumptions}{Assumptions}[section]

\theoremstyle{remark}
\newtheorem*{remark*}{Remark}

\newtheorem*{claim*}{Claim}

\begin{document}
\title{Moment Expansions of the Energy Distance}
\author{
    Ian Langmore \\
    \small Gridmatic Inc.
    \small \texttt{ianlangmore@gmail.com}
}
\date{}
\maketitle
\abstract{
    The energy distance is used to test distributional equality, and as a loss function in machine learning.
    While $\D(X, Y)=0$ only when $X\sim Y$, the sensitivity to different moments is of practical importance.
    This work considers $\D(X, Y)$ in the case where the distributions are close.
    In this regime, $\D(X, Y)$ is more sensitive to differences in the means $\Xbar-\Ybar$, than differences in the covariances $\Delta$. This is due to the structure of the energy distance and is independent of dimension. The sensitivity to on versus off diagonal components of $\Delta$ is examined when $X$ and $Y$ are close to isotropic. Here a dimension dependent averaging occurs and, in many cases, off diagonal correlations contribute significantly less. Numerical results verify these relationships hold even when distributional assumptions are not strictly met.
}

\vspace{0.5cm}\noindent

\medskip
\section{Introduction}

We are motivated by a forecaster tuning forecast $X_t$, hopefully distributed the same as ground truth $Y_t$.
Only one sample $y_t$ is available at a time, so many distributional distances, such as Wasserstein, are impossible to compute. The forecaster may however compute the \emph{energy score}~\cite{Gneiting2007-ms},
\begin{align*}
    \ES(X_t, y_t) &= \Exparg{X}{\|X_t - y_t\|} - \frac{1}{2}\Exparg{X,X'}{\|X_t - X_t'\|}.
\end{align*}
Above, $X_t, X_t'$ are~\iid ~In practice, the expectation is approximated with the $N$ sample estimate,
\begin{align*}
    \widehat{\ES}(t) :&= \frac{1}{N}\sum_{n=1}^N \left\{ \|X_t^{(n)} - y_t\| \right\}
    - \frac{1}{2N(N-1)}\sum_{\substack{n,m=1\\ n\neq m}}^N \left\{ \|X_t^{(n)} - X_t^{(m)}\| \right\}.
\end{align*}
By averaging over times $t\in\calT$ we have a goodness of fit or loss function
\begin{align}
    \label{align:average-energy-score}
    \frac{1}{|\calT|} \sum_{t\in T} \widehat{\ES}(X_t, y_t).
\end{align}
The forecast minimizing~\eqref{align:average-energy-score} can be found with stochastic gradient descent~\cite{Bottou2018-xm}. This strategy has proven successful in machine learning for weather forecasts.\cite{Pacchiardi2024-dg,Kochkov2024-we}.

We drop the time dependence in order to focus on moments. The expected, time independent energy score is
\begin{align*}
    \Exparg{X, Y}{\|X - Y\|} - \frac{1}{2}\Exparg{X,X'}{\|X - X'\|}.
\end{align*}
This is a strictly proper scoring rule~\cite{Gneiting2007-ms}, and is minimized just when $X$ is distributed the same as $Y$ (written $X\sim Y$).

The minimizing forecast does not change if we subtract $(1/2)\E\|Y - Y'\|$. Doing this gives us the (squared) \emph{energy distance}~\cite{Szekely2013-qh}
\begin{align}
    \D(X, Y)
    :&= \E\|X - Y\| - \frac{1}{2}\E\|X - X'\| - \frac{1}{2}\E\|Y - Y'\|.
    \label{align:energy-distance-kernel-form}
\end{align}
One can show that, with
\begin{align*}
  F(\omega) :&= \Exp{e^{i\omega\cdot X}},\qbox{and} G(\omega) := \Exp{e^{i\omega\cdot Y}},
\end{align*}
\begin{align}
  \label{align:energy-dist-fourier}
  \D(X, Y) &= \frac{1}{c_d} \int\frac{|F(\omega) - G(\omega)|^2}{\|\omega\|^{d+1}}\domega,
\end{align}
where
\begin{align*}
  c_d :&= \Vol\left( \dponesphere \right)
  = \frac{2\pi^{(d+1)/2}}{\Gamma\left( \frac{d+1}{2} \right)}
\end{align*}
Since characteristic functions uniquely define distributions, this shows $\D(X, Y) = 0$ if and only if $X$ is equal in distribution to $Y$, written $X\sim Y$. Rigorous conditions on the kernel under which this occurs have been studied by~\cite{Sejdinovic2013-ho}.
$\D$ appears to have been first introduced by~\cite{Szekely2003-zy,Szekely2013-qh} as an extension of the \emph{Cram\'er-von Mises distance}. It appears to have been independently discovered by~\cite{Baringhaus2004-zq}, who provides an alternate proof.

A practical consideration for learning is, what trade-offs will the minimizing forecast make, in light of the fact that $X\sim Y$ generally cannot be achieved? We consider the case where moments of $X - Y$ are much smaller than the average scale $\lambda$ of $X$ and $Y$. Proposition~\ref{proposition:leading-moments} will show that, roughly speaking,
\begin{align}
  \label{align:approx-expansion-for-intro}
  \D(X, Y) &\approx C_1\frac{m_1(\mu)}{\lambda} + C_2\frac{m_2(\Delta)}{\lambda^3},
\end{align}
where $m_i$ are functionals of the difference of means $\mu:=\E X - \E Y$ and difference of covariances $\Delta := \Cov{X} - \Cov{Y}$. Since an optimization will tend to balance the terms in \eqref{align:approx-expansion-for-intro}, we expect a trained model to have $m_1\sim m_2/\lambda\ll m_2$. In other words, using the energy distance as a loss leads to matching the mean more precisely than the second moment.
The most explicit expansion comes in the nearly spherical case, where $m_1(\mu) = \|\mu\|^2$ and $m_2(\Delta) = 2\|\Delta\|_F^2 + \trace(\Delta)^2$, plus a term depending on skew. This is demonstrated numerically for normal and non-normal distributions.

Similar work has considered asymptotic (in sample size) expansions of kernel distances.
In~\cite{Gretton2012-fp} the maximum distance between $X$ and $Y$ detectable with $t$ samples is shown to be proportional to $t^{-1/2}$. The recent work of \cite{Yan2021-yh} obtains an expansion of the test statistic, which reveals which moments are detectable given a sample size. This is done assuming $X$ and $Y$ are affine transformations of~\iid~variables. This is in some sense more comprehensive than our work, but this comes at a cost of complexity of presentation and the affine assumption.

Section~\ref{section:m-dependent-sequences} demonstrates \eqref{align:approx-expansion-for-intro} in a simplified setting. The moment expansion is presented, discussed, and proven, in Section~\ref{section:leading-moments}. Numerical verification is done in Section~\ref{section:numerical-verification}.

\section{Intuition from M-dependent sequences}
\label{section:m-dependent-sequences}
In this simplified setting of limited correlation length, we will see that off diagonal correlation differences $\Delta_{mn}$ for $n\neq m$, appear in a corrector $O(1/d)$ times smaller than the leading term in $\D$. This leading term involves only the first two moments of the marginals, with the difference of means $\mu$ dominating the difference of variance.

Suppose $\left\{ X_n \right\}$ and $\left\{ Y_n \right\}$ are stationary second order sequences with
$X_n\indep Y_\ell$ for all $n,\ell$.
As always, we let $\left\{ X'_n \right\}$ and $\left\{ Y'_n \right\}$ be sequences independent from, but identical in distribution to, $\left\{ X_n \right\}$, $\left\{ Y_n \right\}$.
Let $\D_d$ be the squared energy distance formed from vectors $(X_1,\ldots, X_d)$, $(Y_1,\ldots,Y_d)$. That is,
\begin{align*}
  \D_d &=
  \Exp{\sqrt{\sum_{n=1}^d(X_n-Y_n)^2}}
  -\frac{1}{2}\Exp{\sqrt{\sum_{n=1}^d(X_n-X'_n)^2}}
  -\frac{1}{2}\Exp{\sqrt{\sum_{n=1}^d(Y_n-Y'_n)^2}}.
\end{align*}

The sequences are M-dependent if there exists $M\in\Nat$ such that $X_n \indep X_{n+k}$ for all $|k|> M$. In this case, with
\begin{align*}
  \gamma^2 :&= \sum_{k=-M}^M \Cov{X_n-Y_n,X_{n+k} - Y_{n+k}},
  \qbox{and}
  \nu^2 := \Exp{(X_1 - Y_1)^2},
\end{align*}
we have a central limit result \cite{Hoeffding1948-kp}
\begin{align*}
  \frac{1}{\sqrt{d}}\left[ \sum_{n=1}^d (X_n-Y_n)^2 - \nu^2 \right]
  \to \calN(0, \gamma^2/d).
\end{align*}
We are therefore justified approximating (via.~the delta method)
\begin{align*}
  \frac{1}{\sqrt{d}}\sqrt{\sum_{n=1}^d (X_n-Y_n)^2}
  &\approx \sqrt{\nu^2 + \eps}
  = \nu \left[ 1 + \frac{\eps}{2\nu} - \frac{\eps^2}{8\nu^2} + O\left( \frac{\eps^3}{\nu^3} \right) \right],
\end{align*}
where $\eps\sim\calN(0, \gamma^2/d)$.

The leading term $\nu$ only involves the marginals. 
Since $\E\eps=0$, the corrector to the leading term is $\nu\E\eps^2/(8\nu^2) = \gamma^2/(8d\nu)$. The corrector does involve off diagonal correlations, but is a factor $1/d$ smaller than the leading term.

Analysis of the leading terms from all three contributions to $\D_d$ can be done assuming the $X_n$ are perturbations of $Y_n$. That is,
\begin{align*}
  \Exp{X_1-Y_1} &= \mu_1
  \qbox{and}
  \Var{Y_1} = \lambda^2 - \delta^2/2,\quad \Var{X_1} = \lambda^2 + \delta^2/2,
\end{align*}
with $\mu_1^2 + \delta^2 \ll \lambda^2$.
In this case, up to a $O(\gamma^2/d)$ perturbation
\begin{align}
  \label{align:M-dependent-asymptotic-result}
  \begin{split}
    \frac{1}{\sqrt{d}}\D_d &\approx
    \sqrt{\Exp{(X_1-Y_1)^2}}
    -\frac{1}{2}\sqrt{\Exp{(X_1-X'_1)^2}}
    -\frac{1}{2}\sqrt{\Exp{(Y_1-Y'_1)^2}} \\
    &=
    \sqrt{\mu_1^2 + 2\lambda^2}
    -\frac{1}{\sqrt{2}}\sqrt{(\lambda^2 + \delta^2/2)}
    -\frac{1}{\sqrt{2}}\sqrt{(\lambda^2 - \delta^2/2)} \\
    &= \frac{1}{2^{3/2}}
    \left[ 
    \frac{\mu_1^2}{\lambda}
    - \frac{\mu_1^4}{8\lambda^3}
    + \frac{\delta^4}{8\lambda^3}
    + O\left(\frac{\mu_1^6 + \delta^6}{\lambda^5} \right)
    \right].
  \end{split}
\end{align}

The form of the score ensures that the difference of means arising in $X - Y$ have no opportunity to cancel with terms coming from $X - X'$ or $Y - Y'$ (each of which is independent of the mean). Thus, so long as we are in the perturbative regime (large $\lambda$), the influence of mean perturbations (here $\mu_1$) on the score will dominate the scale perturbations (here $\delta$), irrespective of dimension. On the other hand, the relative influence of on versus off diagonal covariance terms was determined by averaging, and so depends on dimension.

\section{Leading moments of the Fourier representation}
\label{section:leading-moments}

\subsection{Statement of results}
\label{section:leading-moments-results}
We begin by writing the characteristic functions $F$, $G$ as
\begin{align*}
  F &= e^U,\qbox{and} G=e^V.
\end{align*}
Next, re-write the Fourier representation of $\D(X,Y)$~\eqref{align:energy-dist-fourier} in terms of sums and differences of $U$ and $V$. That is, with
\begin{align*}
      \Htilde(\omega) :&=
      \left|
      \exp\left\{ \frac{U(\omega) + V(\omega)}{2} \right\}
      \right|^2,
      \qbox{and}
      W := U - V,
\end{align*}
the squared energy distance is
\begin{align}
  \begin{split}
    \D(X,Y) &= \frac{1}{c_d} \int 
    \frac{\Htilde(\omega)}{\|\omega\|^{d+1}}
     \left|
     e^{W(\omega)/2} - e^{-W(\omega)/2}
     \right|^2
     \domega.
  \end{split}
  \label{align:energy-dist-cumulants}
\end{align}
Our fundamental assumptions will be on the decay of $\Htilde$, which roughly corresponds to relative size of $X$, $Y$ versus the difference of distributions. Then, we will use cumulant identities
\begin{align}
  \label{align:cumulant-identities}
  \begin{split}
    \frac{\p}{\p\omega_n} W\bigg\vert_{\omega=0} &= i\mu_n := i\Exp{(X_n - Y_n)}, \\
    \frac{\p^2}{\p\omega_n\p\omega_m} W\bigg\vert_{\omega=0} &= -\Delta_{mn} := -\left[ \Cov{X_n,X_m} - \Cov{Y_n,Y_m} \right], \\
    \frac{\p^3}{\p\omega_n\p\omega_m\p\omega_\ell} W\bigg\vert_{\omega=0}
    &= -i\kappa_{mn\ell} := -i\,
    \bigg[
    \Exp{(X_m-\Xbar_m) (X_n-\Xbar_n) (X_\ell - \Xbar_\ell)} \\
    &\qquad\qquad\qquad\qquad
      -\Exp{(Y_m-\Ybar_m) (Y_n-\Ybar_n) (Y_\ell - \Ybar_\ell)}
    \bigg],
  \end{split}
\end{align}
where $\Xbar := \Exp{X}$, $\Ybar := \Exp{Y}$,
then write the Taylor series of the exponentiated $W$ in terms of moments and a decay parameter $\lambda$.

In what follows, $A \lesssim B$ means there exists $C > 0$, independent of $\lambda$, such that $A \leq C B$.
\begin{assumptions}[Assumptions for leading moments]\mbox{}\label{assumptions:leading-moments}
  \begin{enumerate}[label = (\roman*)]
    \item The dimension $d>1$
    \item Rapid decay: $\Htilde(\omega) = H(\lambda\omega)$ for $\lambda> 0$.
    \item $C^5$ bounds in the $r-$ball: Define the $C^5_r$ norm using multi-index notation as
      \begin{align*}
        \|f\|_{C^5_r} :&= \sup_{\|\omega\|\leq r} \sum_{|\alpha|\leq 5} |\p^\alpha f(\omega)|.
      \end{align*}
      Then we assume
      \begin{align*}
        \| e^{W} \|_{C^5_r} + \|e^{-W}\|_{C^5_r}&\leq B(r),
      \end{align*}
      where, with $\dOmega$ the measure on $\dsphere$,
      \begin{align*}
        \int \frac{H(\lambda \omega)}{\|\omega\|^{d+1}}\|\omega\|^5 B(\|\omega\|)\domega
        &= \frac{1}{\lambda^{4}}\int_{\dsphere}\int_0^\infty H(r\theta)r^{5-2} B(r/\lambda)\dr\dOmega(\theta) \\
        &\lesssim \frac{1}{\lambda^{4-b}},
      \end{align*}
      for some $b\in(0, 1)$.
    \item (Optional) Spherical symmetry: $H(\omega) =  h(\|\omega\|)$
  \end{enumerate}
\end{assumptions}

\begin{proposition}[Taylor expansion]\mbox{}\\
  Given (i)-(iii) of assumptions \ref{assumptions:leading-moments},
  \begin{align*}
    \D(X, Y)
    &=\frac{1}{\lambda} \frac{1}{c_d}
    \int_0^\infty \int_{\dsphere}H(r\theta)
    (\theta\cdot\mu)^2
    \dOmega(\theta)\dr
    \\
    &+\frac{1}{\lambda^3}\frac{1}{12 c_d}
    \int_0^\infty r^2\int_{\dsphere}H(r\theta)
    \left[
    3(\theta\cdot\Delta\theta)^2
    - (\theta\cdot\mu)^4
    - 4
    (\theta\cdot\mu)
     \sum_{ijk=1}^d \kappa_{ijk}\theta_i\theta_j\theta_k
    \right]
    \dOmega(\theta)\dr
    \\
    &+ R(\lambda), \\
    |R(\lambda)| &\lesssim \frac{1}{\lambda^{4-b}}.
  \end{align*}

Under spherical symmetry, assumption (iv), we have a more explicit form
  \begin{align*}
    \D(X, Y)
    &=\frac{1}{\lambda} \frac{\Vol\left( \dsphere \right)}{c_d d} \|\mu\|^2 \int_0^\infty h(r)\dr
    \\
    &+ \frac{1}{\lambda^3}\frac{\Vol\left( \dsphere \right)}{4 c_d d(d+2)}
    \left[
      2\|\Delta\|_F^2 + \trace\left( \Delta \right)^2 - \|\mu\|^4
      - 4\beta\cdot\mu
    \right]\int_0^\infty h(r)r^2\dr
    \\
    &+ R(\lambda).
  \end{align*}
  where
  \begin{align*}
    \beta\cdot\mu :&= \sum_{i,j=1}^d\kappa_{iij}\mu_j
    = \sum_{i,j=1}^d \Exp{(X_i - \Xbar_i)^2(X_j-\Xbar_j) - (Y_i-\Ybar_i)^2(Y_j-\Ybar_j)}
    \left( \Xbar_j - \Ybar_j \right)
  \end{align*}
  \label{proposition:leading-moments}
  measures the alignment between difference of skewness and difference of the mean, since $\Exp{(X_i-\Xbar_i)^2(X_j-\Xbar_j)}$ is zero if $X$ is symmetric about its mean.
\end{proposition}

Using the asymptotic relation, $\Gamma(x + \alpha)\sim \Gamma(x)x^\alpha$, for large $x$, we have
\begin{align*}
  \frac{\Vol\left( \dsphere \right)}{c_d} &\sim \sqrt{\frac{d}{2\pi}}.
\end{align*}
allowing the asymptotic form
\begin{corollary}
  Given (i) - (iv) of assumptions \ref{assumptions:leading-moments},
  \begin{align*}
    \D(X, Y)
    &\sim \frac{1}{\lambda}\frac{1}{\sqrt{d2\pi}}\|\mu\|^2 \int_0^\infty h(r)\dr
    \\
    &+ \frac{1}{\lambda^3}\frac{1}{4d}\frac{1}{\sqrt{d2\pi}}
    \left[
    2\|\Delta\|_F^2 + \trace\left( \Delta \right)^2 - \|\mu\|^4
    - 4\beta\cdot\mu
  \right]
  \int_0^\infty h(r)r^2\dr,
  \end{align*}
  where $\sim$ indicates equality up to higher order terms in $1/\lambda$ and $1/d$.
  \label{corollary:leading-moments-asymptotics}
\end{corollary}

\subsection{On satisfying assumptions~\ref{assumptions:leading-moments}}
\label{section:satisfying-assumptions}

The decay assumption (ii), $\Htilde(\omega)=H(\lambda\omega)$, roughly corresponds to $X$, $Y$ being much larger than their difference. To see this, rescale $X\mapsto\lambda X$ and we see the characteristic function changes from $F(\omega)$ to $F_\lambda(\omega)= F(\lambda\omega)$. Similarly suppose $Y$ has characteristic function $G(\lambda\omega)$. Then, referring to~\eqref{align:energy-dist-cumulants},
\begin{align*}
    \Htilde(\omega) &= |F(\lambda\omega)G(\lambda\omega)| =: H(\lambda\omega).
\end{align*}
On the other hand,
\begin{align*}
    \left|
    e^{W(\omega)/2} - e^{-W(\omega)/2}
    \right|
    &= 
    \left|
    \sqrt{\frac{F(\lambda\omega)}{G(\lambda\omega)}}
    -
    \sqrt{\frac{G(\lambda\omega)}{F(\lambda\omega)}}
    \right|,
\end{align*}
and if $F\approx G$, this term should be $\sim O(1)$.

Next consider condition (iii).
Since the derivatives of the characteristic function at zero are the moments, the $C^5_r$ bounds imply the first five moments of $X$ and $Y$ exist.
the growth condition $\leq B(r)$ must be checked case by case. For example, if $Y\sim\calN(0, \lambda^2)$ and $X\sim\calN(0, \lambda^2 + \delta^2)$, then 
\begin{align*}
   h(r) &= e^{-r^2(\lambda^2 + \delta^2/2)}, \\
  e^{-W(r)} &\leq e^{r^2\delta},
  \qbox{and}
  \|e^{-W}\|_{C^k_r} \lesssim (1 + r^k)e^{r^2\delta^2},
\end{align*}
which satisfies assumptions~\ref{assumptions:leading-moments} (iii) with $b=\delta^2/\lambda^2$.

Finally, since it is standard practice to whiten variables before hypothesis testing or optimization, the spherical assumption $(iv)$ will be reasonable in some cases.



\subsection{Interpretation and applications}
\label{section:leading-moments-interpretation}
Proposition~\ref{proposition:leading-moments} sheds light on the energy score in the perturbative regime (large $\lambda$).
First, notice that differences of the mean have leading influence due to the structure of the score, which can be written using Taylor series of cumulants, weighted by a decay of $\|\omega\|^{-(d + 1)}$. The order of $\lambda$ attached to each term of the series determines the relative influence, and is independent of dimension. Second, the influence of on versus off diagonal members of the covariance is determined by averaging. Specifically, spherical averages result in the term $2\|\Delta\|_F^2 + \trace(\Delta)^2$. This will involve dimension, but the story is nuanced as we will see below.

\subsubsection{Multivariate Gaussians and the M-dependent case revisited}
We can use the form in corollary~\ref{corollary:leading-moments-asymptotics} to obtain the leading terms when $Y\sim\calN(0, \lambda^2 I - \Delta/2)$, $X\sim\calN(\mu, \lambda^2 I + \Delta/2)$. Then
\begin{align*}
    \Htilde(\omega)
  = H(\lambda\omega)
  &= \exp\left\{ -\omega\cdot\left( \frac{2I}{2\lambda^2} \right)\omega \right\}
  = \exp\left\{ -\|\omega/\lambda\|^2 \right\}
\end{align*}
This means $ h(r) = \exp\left\{ -r^2 \right\}$ and then
\begin{align*}
  \int_0^\infty h(r)\dr = \frac{\sqrt{\pi}}{2},
  \qbox{and}
  \int_0^\infty h(r)r^2\dr = \frac{\sqrt{\pi}}{4}.
\end{align*}
Using the symmetry of $X$, $Y$ about their mean (which implies $\beta\equiv0$), and
corollary~\ref{corollary:leading-moments-asymptotics}, we have
\begin{align}
  \label{align:multivariate-normal-leading-moments}
  \begin{split}
    \D(X, Y)
    &\sim \frac{1}{\lambda}
    \frac{1}{(8d)^{1/2}}
    \|\mu\|^2
    + \frac{1}{\lambda^3}
    \frac{1}{8d}
    \frac{1}{(8d)^{1/2}}
    \left[ 2\|\Delta\|_F^2 + \trace\left( \Delta \right)^2 - \|\mu\|^4 \right]
  \end{split}
\end{align}

To plug in the M-dependent case of Section~\ref{section:m-dependent-sequences}, let us choose $\Delta$ to have $\delta^2$ on the diagonal and $\rho^2$ on the $M$ sub/super diagonals. Then, for large $d$,
\begin{align*}
  \|\mu\|^2 &= d \mu_1^2, \qbox{and} \|\mu\|^4 = d^2\mu_1^4, \\
  \trace\left( \Delta \right)^2 &= d^2\delta^4,
  \qbox{and}
  \|\Delta\|_F^2 \sim d\delta^4 + 2dM\rho^4,
\end{align*}
so that
\begin{align}
  \label{align:M-dependent-normal-leading-moments}
    \D(X, Y)
    &\sim
    \frac{1}{\lambda}\sqrt{\frac{d}{8}}\mu_1^2
    + \frac{1}{\lambda^3}\sqrt{\frac{d}{8}} \left[ 
      \frac{\delta^4}{8} - \frac{\mu_1^4}{8}
      + \frac{\delta^4}{4d} + \frac{M\rho^4}{2d}
    \right]
\end{align}
This more precise result~\eqref{align:M-dependent-normal-leading-moments} agrees with~\eqref{align:M-dependent-asymptotic-result}. In particular, we see that off-diagonal covariance contribution $M\rho^4/(2d)$ is $\sim 1/d$ smaller than the leading diagonal contribution $\delta^4/8$. This is due to $\trace(\Delta)^2$ giving an $O(d)$ weight to each individual diagonal element, whereas the $\|\Delta\|_F^2$ term gives a weight of one to all elements.

\subsubsection{Cosine similarity of gradients}
\label{section:cosine-similarity-of-gradients}
In the machine learning setting, $X$ is a forecast from a learned model, and $Y$ is ground truth. An ideal loss term to encourage correct covariance structure would be
\begin{align*}
  \calI :&= \|\Cov{X} - \Cov{Y}\|_F^2 = \|\Delta\|_F^2 = \sum_{n=1}^d \lambda_n^2,
\end{align*}
where $\lambda_n$ is the $n^{th}$ eigenvalue of the difference of covariance matrices $\Delta$. This weights every discrepancy in the covariance difference equally, and penalizes the spectrum in the L2 sense. However, when using $\D(X, Y)$ as a loss, the leading contribution from covariance is
\begin{align*}
  \calL :&= 2\|\Delta\|_F^2 + \trace(\Delta)^2
  = 2\sum_{n=1}^d\lambda_n^2 + \sum_{n,m=1}^d \lambda_n\lambda_m.
\end{align*}

Since both $\calI$ and $\calL$ have a minimum at $\Delta=0$, it's fair to ask how similar $\calL$ is to $\calI$.

Let $\xi$ be the model parameters, for example neural network weights. Typically $\xi$ is learned by a flavor of gradient descent. Since the $\lambda_n$ are a function of $\xi$, and so we can express the gradients as
\begin{align*}
  \frac{\p \calL(X, Y)}{\p\xi} &= \sum_{n=1}^d\frac{\p\lambda_n}{\p\xi}\frac{\p\calL}{\p\lambda_n},
  \qbox{and}
  \frac{\p \calI(X, Y)}{\p\xi} = \sum_{n=1}^d\frac{\p\lambda_n}{\p\xi}\frac{\p\calI}{\p\lambda_n}.
\end{align*}
We are thus motivated to study the cosine similarity of the gradients $\nabla_\lambda\calL$ vs. $\nabla_\lambda\calI$. We have
\begin{align}
  \begin{split}
  \calS :&= \frac{\nabla_\lambda\calL\cdot\nabla_\lambda\calI}{\left\|\nabla_\lambda\calL\right\|\,\left\|\nabla_\lambda\calI\right\|}
  = \frac{2 + \gamma^2}{\sqrt{4 + \gamma^2(4 + d)}}.
  \end{split}
  \label{align:cosine-similarity-of-grads}
\end{align}
where $\gamma^2 := \trace(\Delta)^2 / \|\Delta\|_F^2\in[0, d]$ is a measure of diagonal importance.


Consider the case of local correlations of size $M$, inside a large system of size $d \gg M\gg 1$ (as in Section~\ref{section:m-dependent-sequences}). This means, $\|\Delta\|_F^2\sim Md$. The diagonal of $\Delta$ will have positive or negative elements if $X$ systematically over/under represents variance. In this case, $\trace(\Delta)^2\sim d^2$, so $\gamma^2\sim d/M$, and then $\calS \sim 1 / \sqrt{M}$.
On the other hand, if the diagonal is unbiased, we can expect $\trace(\Delta)^2\sim d$, so $\gamma^2\sim 1 / M$, and then $\calS\sim\sqrt{M/d}$.
The effect on learning should be that the componentwise scales of $X$ become equal to those of $Y$ fairly quickly. Once that happens, the gradients of $\calL$ become close to orthogonal to those of $\calI$, and the off-diagonal correlations may be neglected.

On the other hand, consider a high resolution system $d\gg1$, and large scale correlation length that is a constant multiple of the system size. That is, $\|\Delta\|_F^2\sim \rho d^2$, for $\rho\in(0, 1)$. If the diagonal is biased, $\gamma^2\sim 1 /\rho$, and then $\calS\sim \sqrt{1/d}$. If on the other hand the diagonal of $\Delta$ is unbiased, $\trace(\Delta)^2\sim d$, so $\gamma^2\sim 1 / (\rho d)$, and then $\calS\sim O(1)$.
As a result, we expect scale bias to distract from learning large scale correlations with energy distance. Once the scale is unbiased, the learning efficiency will be much better.

These four scenarios are summarized in table~\ref{tab:cosine-distance-cases}.

\begin{table}[htbp] 
    \captionsetup{width=0.9\linewidth}
    \centering 
    \begin{tabular}{l|c|c}
        \toprule
        & \textbf{\parbox{4.5cm}{\centering Local correlations $\|\Delta\|_F^2\sim Md$}}
        & \textbf{\parbox{4.5cm}{\centering Global correlations $\|\Delta\|_F^2\sim\rho d^2$}} \\
        \midrule
        \textbf{\parbox{4.5cm}{\centering Biased scales\\$\trace(\Delta)^2\sim d^2$}}
        & $\calS\sim\sqrt{1/M}$
        & $\calS\sim\sqrt{1/d}$ \\
        \midrule
        \textbf{\parbox{4.5cm}{\centering Unbiased scales\\$\trace(\Delta)^2\sim d$}}
        & $\calS\sim\sqrt{M/d}$
        & $\calS\sim1$ \\
        \bottomrule
    \end{tabular}
    \caption{
        \textbf{Asymptotic similarity} implied by~\eqref{align:cosine-similarity-of-grads} as dimension $d\to\infty$.
        The local correlation length $M\gg1$, or, alternatively the global correlation fraction $\rho\in(0,1)$, are both fixed.
        A machine learning practitioner can use the table by choosing a situation (local vs.~global correlations), then assume scales start off biased, and move towards unbiased as training progresses.
        Smaller $\calS$ means gradient descent corrects off-diagonal correlations slower.
    }
    \label{tab:cosine-distance-cases} 
\end{table}

\subsection{Proof of proposition~\ref{proposition:leading-moments}}
\label{section:leading-moments-proof}

Since $F(\omega) = \Exp{e^{i\omega\cdot X}}$, we are tempted to expand $F$, $G$ in power series inside the Fourier representation \eqref{align:energy-dist-fourier}. Doing this directly does not work, since every term would diverge. For that reason, we rewrite the distance as~\eqref{align:energy-dist-cumulants}, so $\Htilde$ can be used as an integrating factor. We then apply the cumulant identities \eqref{align:cumulant-identities}.

The Taylor expansion
\begin{align}
  \label{align:taylor-expansion-in-proof}
  \begin{split}
    \D(X, Y)
    &=\frac{1}{c_d}\sum_{k=1}^4\int_0^\infty \frac{r^{k-2}}{k!\lambda^{k-1}}\int_{\dsphere}H(r\theta) \psi_k(\theta)\dOmega(\theta)\dr
    + R(\lambda), \\
    |R(\lambda)| &\lesssim \frac{1}{\lambda^{4-b}},
  \end{split}
\end{align}
immediately follows from assumptions (ii), (iii). What remains is to identify the terms $\psi_k$, and integrate them under the spherical symmetry condition.

Readers take note that, if we were only interested in spherical symmetry, a simpler proof is available using Pizetti's theorem~\cite{Li2011-dz}.

\begin{lemma}
  Under assumptions~\ref{assumptions:leading-moments} (i)-(iii),
  \label{lemma:psik-values}
  \begin{align*}
    \psi_1(\theta) &= \psi_3(\theta) \equiv 0, \\
    \psi_2(\theta) &= 2 (\theta\cdot\mu)^2, \\
    \psi_4(\theta) &= 6 (\theta\cdot\Delta\theta)^2 - 2(\theta\cdot\mu)^4
    -8
    (\theta\cdot\mu)
     \sum_{ijk=1}^d \kappa_{ijk}\theta_i\theta_j\theta_k
    .
  \end{align*}
\end{lemma}
\begin{proof}[Proof of lemma \ref{lemma:psik-values}]
  Put
  \begin{align*}
    S &= e^{W/2} - e^{-W/2},
    \quad S' = \frac{\diff}{\dW}S,
    \quad S''=\frac{\diff^2}{\dW^2}S,
    \qbox{and so on \ldots}.
  \end{align*}
  Let $\Wbar$, $\Sbar$ denote the complex conjugate of $S$, $W$.
  Also define differential operators
  \begin{align*}
    \p_i :&= \frac{\p}{\p\omega_i},
    \quad
    \p_{ij} := \frac{\p^2}{\p\omega_i\p\omega_j},
    \qbox{and so on \ldots} \\
    W_i &= \p_iW,
    \quad
    W_{ij} = \p_{ij}W,
    \qbox{and so on \ldots}
  \end{align*}
With this notation, we will have
\begin{align*}
    \psi_1(\theta) &= \sum_i \theta_i \p_i S\Sbar\big\vert_{\omega=0},
    \quad
    \psi_2(\theta) = \sum_{i,j}\theta_i\theta_j\p_{ij}S\Sbar\big\vert_{\omega=0},
    \\
    \psi_3(\theta) &= \sum_{i,j,k}\theta_i\theta_j\theta_k\p_{ijk}S\Sbar\big\vert_{\omega=0},
    \quad
    \psi_4(\theta) = \sum_{i,j,k,\ell}\theta_i\theta_j\theta_k\theta_\ell\p_{ijk\ell}S\Sbar\big\vert_{\omega=0}.
\end{align*}

Starting at $\psi_1$, we see
\begin{align*}
  \p_iS\Sbar &= S'\Sbar W_i + S\Sbar'\Wbar_i
  = 2 \Real\left\{ S'\Sbar W_i \right\}.
\end{align*}
Since $S(0)=0$, $\psi_1\equiv0$.

As for $\psi_2$,
\begin{align*}
  \p_{ij}S\Sbar &= 2\Real\left\{ 
    S''\Sbar W_iW_j
  + S'\Sbar'W_i\Wbar_j + S'\Sbar W_{ij}
\right\}.
\end{align*}
Now use $S(0)=\Sbar(0)=0$, $S'(0)=\Sbar'(0)=1$, and \eqref{align:cumulant-identities} to see
\begin{align*}
  \psi_2(\theta) &=
  2 \Real\left\{
    \sum_{ij} (i\mu_i\theta_i)(-i\mu_j\theta_j) 
  \right\}
  = 2(\mu\cdot\theta)^2.
\end{align*}

Proceeding to $\psi_3$,
\begin{align*}
  \p_{ijk}S\Sbar &= 2\Real\left\{ 
       S'''\Sbar W_iW_jW_k
       + S''\Sbar'W_iW_j\Wbar_k
       + S''\Sbar W_{ik}W_j
       + S''\Sbar W_iW_{jk}
  \right. \\
  &\qquad 
  + S''\Sbar'W_i\Wbar_jW_k
  + S'\Sbar'' W_i\Wbar_j\Wbar_k
  + S'\Sbar' W_{ik}\Wbar_j 
  + S'\Sbar' W_i\Wbar_{jk}
  \\
  &\qquad
  + S''\Sbar W_{ij}W_k
  + S'\Sbar' W_{ij}\Wbar_k
  + S'\Sbar W_{ijk}
  \left.
  \right\}
\end{align*}
Now using $S''(0)=0$, $S'''(0) = 1/4$, all but terms like $S'\Sbar' W_{ij}\Wbar_k$ drop out. Taking the real part, these drop out as well, and $\psi_3\equiv0$.

As for $\psi_4$, we will write down only terms that involve odd powers of $S$, since the even ones will become zero. This leads to
\begin{align*}
  \p_{ijk\ell}S\Sbar \bigg\vert_{\omega=0} &= 2\Real\left\{ 
    \frac{1}{4}\left( 
      W_iW_jW_k\Wbar_\ell
      + W_iW_j\Wbar_kW_\ell
      + W_i\Wbar_jW_kW_\ell
      + W_i\Wbar_j\Wbar_k\Wbar_\ell
    \right)
  \right.
  \\
  &\qquad
  + W_{ik\ell} \Wbar_j + W_{ik}\Wbar_{j\ell}
  + W_{i\ell} \Wbar_{jk} + W_i \Wbar_{jk\ell}
  + W_{ij\ell}\Wbar_k + W_{ij} \Wbar_{k\ell}
  \\
  &\qquad\left.
  + W_{ijk}\Wbar_\ell
  \right\}
\end{align*}
After applying the cumulant identities \eqref{align:cumulant-identities}, and exchanging dummy indices, we have the desired result.
\end{proof}

\begin{lemma}[Spherical integrals]
  \begin{align*}
    \int_\dsphere (\theta\cdot\mu)^2\dOmega(\theta) &= \|\mu\|^2 \frac{\Vol(\dsphere)}{d}, \\
    \int_\dsphere (\theta\cdot\mu)^4\dOmega(\theta) &= 3 \|\mu\|^4 \frac{\Vol(\dsphere)}{d(d+2)}, \\
    \int_\dsphere (\theta\cdot\Delta\theta)^2\dOmega(\theta) &= \left[2\|\Delta\|_F^2 + \trace(\Delta)^2\right] \frac{\Vol(\dsphere)}{d(d+2)}, \\
    \int_\dsphere
    (\theta\cdot\mu)
     \sum_{ijk=1}^d \kappa_{ijk}\theta_i\theta_j\theta_k
    \dOmega(\theta)
    &= 
    \left[ 3\sum_{i,j=1}^d \kappa_{iij}\mu_j \right]\frac{\Vol(\dsphere)}{d(d+2)}.
  \end{align*}
  \label{lemma:spherical-integrals}
\end{lemma}
\begin{proof}[Proof of lemma~\ref{lemma:spherical-integrals}]
  We show how to reduce the final term, since this is the most difficult. Here we will use that $\kappa_{\sigma(i)\sigma(j)\sigma(k)}$ is invariant under any permutation $\sigma$.
  First start with
  \begin{align*}
    \int_\dsphere 
    \left( \sum_{ijk}\theta_i\theta_j\kappa_{ijk}\theta_k \right)
    \left( \sum_\ell\theta_\ell\mu_\ell \right) \dOmega(\theta),
    &=
    \sum_{ijk\ell}\kappa_{ijk}\mu_\ell\int_\dsphere \theta_i\theta_j\theta_k\theta_\ell\dOmega(\theta).
  \end{align*}
  By symmetry, the only summands that survive are ones where each index is paired with (at least) one other. This gives
  \begin{align}
    \label{align:spherical-integrals-expanded-term}
    \sum_{\substack{i=1}}^d\kappa_{iii}\mu_i\int_\dsphere \theta_i^4\dOmega(\theta) +
    3\sum_{\substack{i,j=1\\i\neq j}}^d\kappa_{iij}\mu_j\int_\dsphere \theta_i^2\theta_j^2\dOmega(\theta).
  \end{align}

  Integrals of polynomials over a sphere are given in~\cite{Folland2001-uq}. In general
  \begin{align*}
    \int_\dsphere \theta_1^{a_1}\cdots\theta_d^{a_d}\dOmega(\theta)
    &= 2\frac{\Gamma\left( \frac{a_1 + 1}{2} \right)\cdots \Gamma(\frac{a_d+1}{2})}{\Gamma(\frac{a_1+1}{2} + \cdots + \frac{a_d+1}{2})}.
  \end{align*}
  This gives
  \begin{align*}
    \int_\dsphere\theta_1^4\dOmega(\theta) &= 3\frac{\Vol(\dsphere)}{d(d+2)},
    \qquad
    \int_\dsphere\theta_1^2\theta_2^2\dOmega(\theta) = \frac{\Vol(\dsphere)}{d(d+2)}.
  \end{align*}
  Inserting into \eqref{align:spherical-integrals-expanded-term} gives the final equality of the lemma. The others are similar.
\end{proof}

\begin{proof}[Proof of proposition~\ref{proposition:leading-moments}]
  Insert the results of lemmas~\ref{lemma:psik-values} and \ref{lemma:spherical-integrals} into \eqref{align:taylor-expansion-in-proof}.
\end{proof}

\section{Numerical verification}
\label{section:numerical-verification}

We verify proposition~\ref{proposition:leading-moments} in the case where $Y$ is a centered member of a location-scale family, and $X$ is a perturbation.
We test how well $\D$ values are determined by corollary~\ref{corollary:leading-moments-asymptotics}, which does assume spherical symmetry.
Note that $H$ will \emph{not} have perfect spherical symmetry, even in the Gaussian case.
That is, we find best fit coefficients $\alpha_1,\alpha_2$ so that
\begin{align}
    \label{align:distance-regression}
  \D(X, Y)
  &\approx \alpha_1 \|\mu\|^2
  + \alpha_2 \left[ 2\|\Delta\|_F^2 + \trace\left( \Delta \right)^2 - \|\mu\|^4 - 4\beta\cdot\mu \right].
\end{align}
We call this our \emph{$\D$ regression}.
For each distribution type, we perform a regression and measure the coefficient of determination $R^2$. In most cases, $R^2\approx1$, which indicates corollary~\ref{corollary:leading-moments-asymptotics} described $\D$ quite well.

In the Gaussian case, $\alpha_k$ can be computed by hand as~\ref{align:multivariate-normal-leading-moments}. So rather than a regression we use~\eqref{align:multivariate-normal-leading-moments} directly, as a \emph{$\D$ estimation}.

Our sweep considers various transformations of multivariate Gaussians. That is, we set
\begin{align*}
    X\sim T_X(Z),\quad Z\sim\calN(\mu, C),
    \qbox{and}
    Y\sim T_Y(Z'),\quad Z'\sim\calN(0, I_d).
\end{align*}
$T_X, T_Y$ are transformations generating different distribution classes. Our theory indicates that $\alpha_1,\alpha_2$ should depend (mostly) on the transformation and dimension. So, for every $(d,T_X, T_Y)$ triplet, we fit one $\alpha_1,\alpha_2$ regression pair.

The transformations and their label are as follows. We set $T_X=T_Y$ to the identity, and use the label \texttt{Gaussian}. We expect~\eqref{align:multivariate-normal-leading-moments} to hold here quite well. For $\nu\in\left\{ 2,3,5 \right\}$, we make our \texttt{MultivariateT(dof=$\nu$)} with $T_X=T_Y=T$, and $T(z) = z/\sqrt{\chi^2_\nu/\nu}$. Note that this distribution has finite moments $k$ only for $k<\nu$. For $\sigma\in\left\{ 0.75, 1, 1.25 \right\}$, we make \texttt{Exp($\sigma$)} with $T_X=T_Y=T$, and $T(z) = \exp\left\{ z\sigma \right\}$. This provides a skew in both $X$ and $Y$. To test skew in the difference, we set $T_Y$ to the identity, and use a $X\mapsto$\texttt{SinhArcsinh(X, skew)} transformation for \texttt{skew}$\in\left\{ 0.05, 0.1, 0.2 \right\}$. This transformation is available in the \texttt{TensorFlow Probability} library~\cite{Dillon2017-zq}.
\begin{figure}[htbp]
  \centering 
  \includegraphics[width=0.9\textwidth]{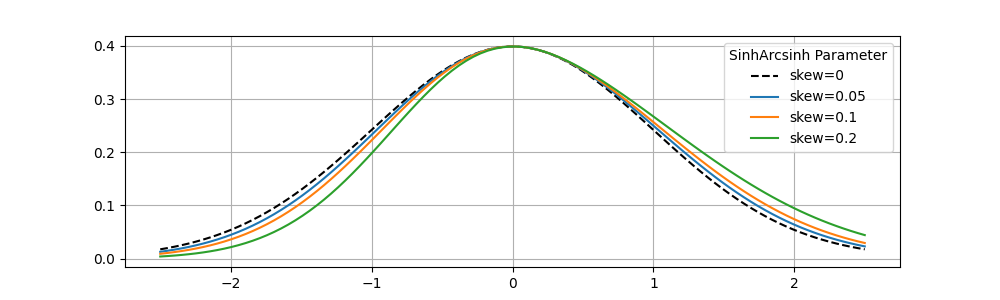}
  \caption{\textbf{Sweep of skewness parameter:} Here we show the probability density resulting from transforming a unit Normal by \texttt{SinhArcsinh(skew)}, for \texttt{skew}$\in\left\{ 0, 0.05, 0.1, 0.2 \right\}$. When \texttt{skew}=0, the transformation is the identity, so the density is the same as the unit Normal. As \texttt{skew} increases, the mass is tilted to the right.}
  \label{fig:sinh-arcsinh-skewness}
\end{figure}

For each of the ten transformations above, we generate results in dimension $d\in\left\{ 16,32,64 \right\}$. For each of these thirty combinations, we use $\mu = \left( \mu_1,\ldots,\mu_1 \right)\in\Rd$, with three diffrent values of $\mu_1$. For each of these ninety combinations, we generate twenty eight correlation matrices as different random rescaling of either Wishart or smooth exponential decays.

Results for Gaussian $X, Y$ are shown with small perturbations in figure~\ref{fig:normal-values-small-pert}, and larger perturbations in figure~\ref{fig:normal-values-large-pert}.
This case is different since we compute the regression coefficients directly from~\eqref{align:multivariate-normal-leading-moments}.
In the figures, the top row in each includes all three $\mu_1$ values.
We see three different clusters corresponding to different $\mu_1$ values, and as theory would predict, since the effect of different $\mu_1$ dominates the score. 
At this a large scale, the fit is quite good, with $R^2\approx1$. The bottom row considers one fixed $\mu_1$ cluster. Here we see that smaller dimension distances are offset from the theoretical estimate.
With larger scaling factor $\sigma$, the \texttt{Exp($\sigma$)} starts deviating more from theory.

Figure~\ref{fig:all-mu-values} shows non-Gaussian results at $d=64$, and every $\mu_1$ value. At this scale, the fit is quite good. 
Figure~\ref{fig:one-mu-values} shows non-Gaussian results with $\mu_1=0.15$ only. This corresponds to the middle cluster in figure~\ref{fig:all-mu-values}.
The \texttt{MultivariateT} covariance does not follow theory once the degrees of freedom $\nu=2$. This is expected, since the second moment becomes infinite.
The \texttt{SinhArcsinh} covariance deviates strongly from spherical covariance theory for larger skews as well. This is expected, since while corollary~\ref{corollary:leading-moments-asymptotics} allowed for skewed differences in moments, it still requires the sum of measures to be spherically symmetric.

\begin{figure}
  \centering
  \includegraphics[width=0.9\textwidth]{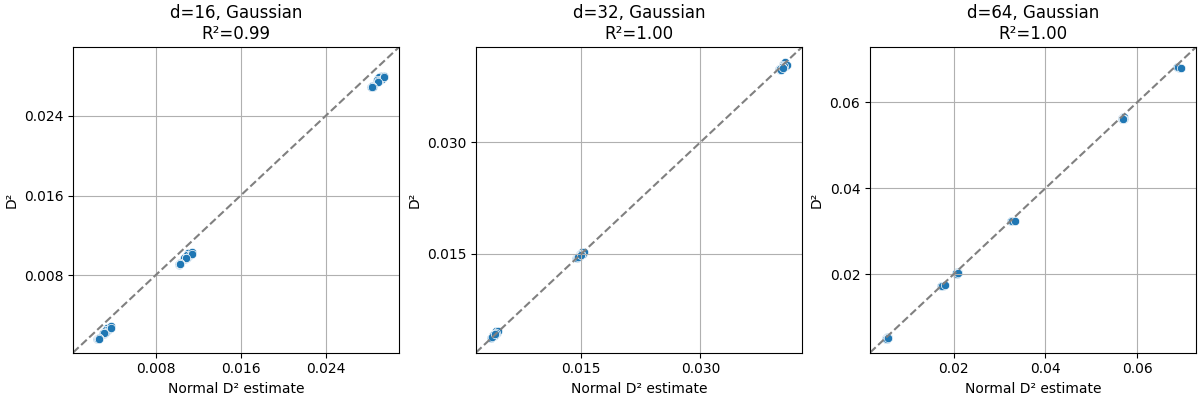}
  \includegraphics[width=0.9\textwidth]{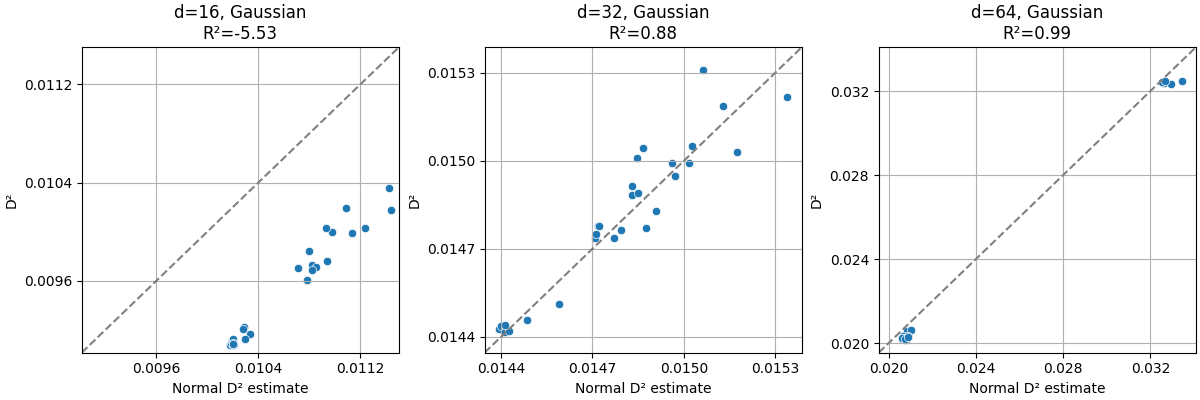}
  \caption{\textbf{Gaussian distributions with small perturbations:} 
      Here $Y\sim\calN(0, I_d)$, and $X\sim\calN(\mu, C)$ is a small perturbation of $Y$. We compare sample $\D(X, Y)$ with the theoretical estimate of~\eqref{align:multivariate-normal-leading-moments}, for $d=16, 32, 64$. \textbf{Top:} Different values of $\mu_1$ lead to different clusters. \textbf{Bottom:} Fixing $\mu_1=0.06$, we see the effect of different covariance only.
  }
  \label{fig:normal-values-small-pert}
\end{figure}
\begin{figure}
  \centering
  \includegraphics[width=0.9\textwidth]{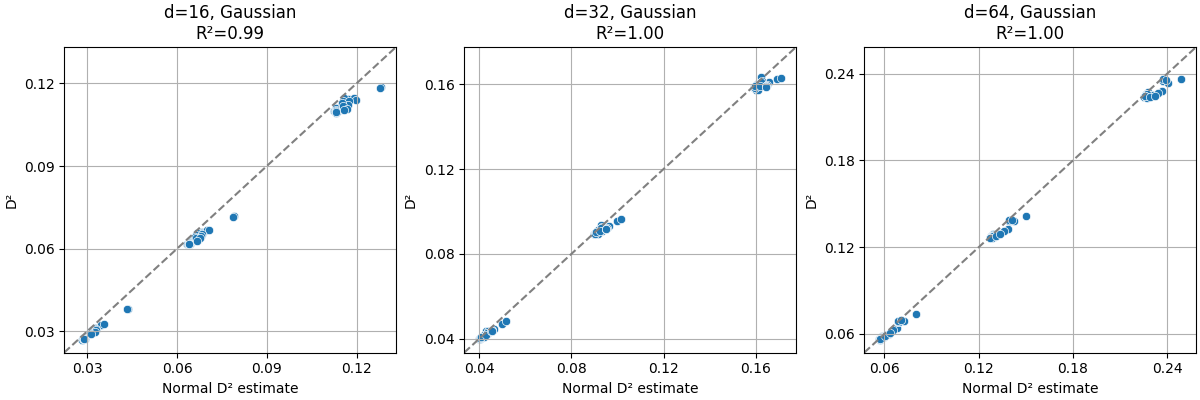}
  \includegraphics[width=0.9\textwidth]{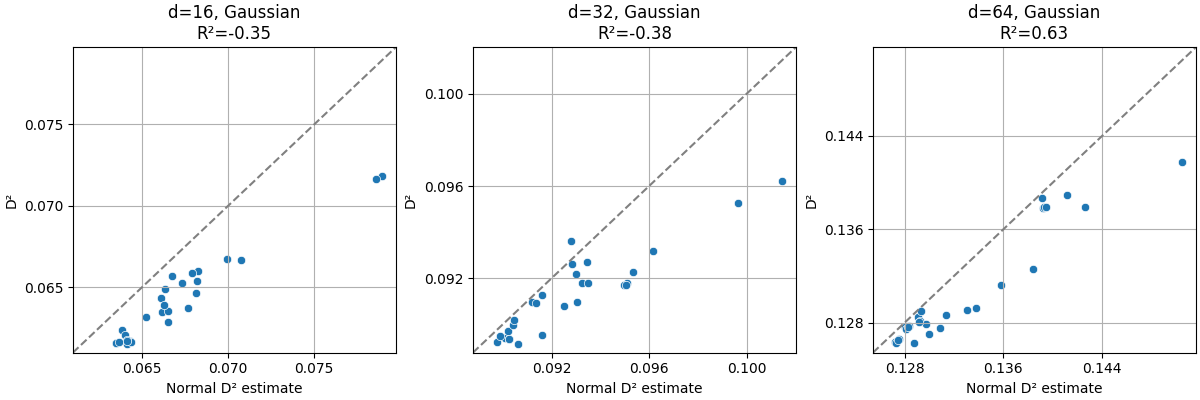}
  \caption{\textbf{Gaussian distributions with larger perturbations:} 
      Here $Y\sim\calN(0, I_d)$, and $X\sim\calN(\mu, C)$ is a larger perturbation of $Y$. We compare sample $\D(X, Y)$ with the theoretical estimate of~\eqref{align:multivariate-normal-leading-moments}, for $d=16, 32, 64$. \textbf{Top:} Different values of $\mu_1$ lead to different clusters. \textbf{Bottom:} Fixing $\mu_1=0.15$, we see the effect of different covariance only. $R^2$ values are lower here than when perturbations were small (figure~\ref{fig:normal-values-small-pert}).
  }
  \label{fig:normal-values-large-pert}
\end{figure}

\begin{figure}
  \centering 
  \includegraphics[width=0.9\textwidth]{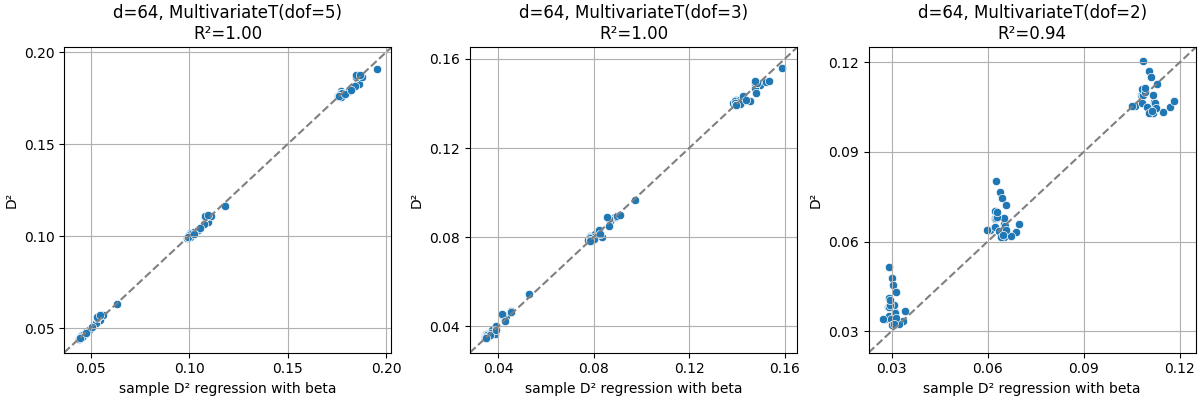}
  \includegraphics[width=0.9\textwidth]{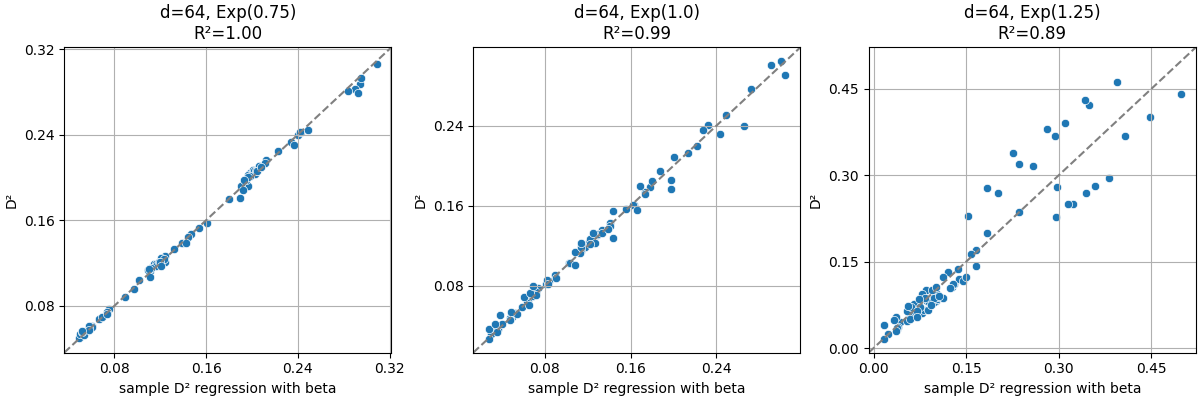}
  \includegraphics[width=0.9\textwidth]{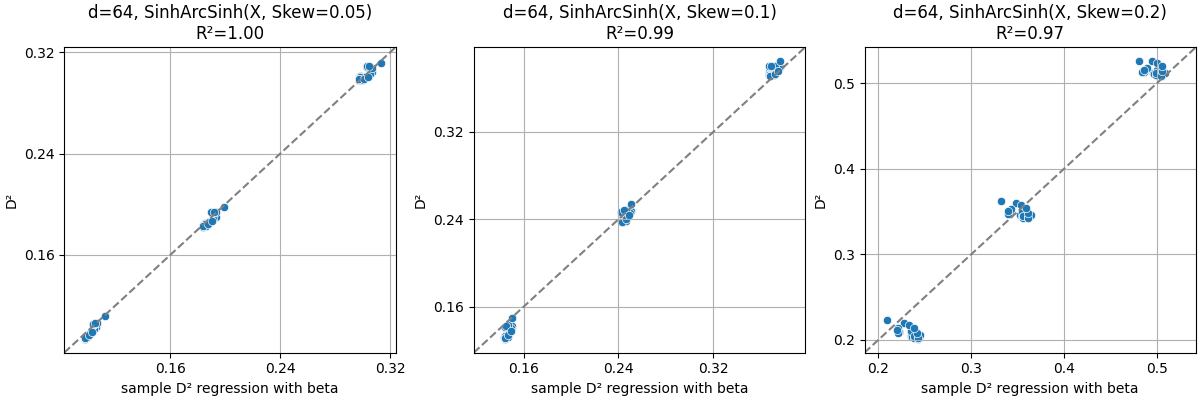}
  \caption{
  \textbf{Non-Gaussian distributions with larger perturbations:} Here $X$ and $Y$ are transformed Gaussians. Only $d=64$ is shown.
  We compare sample $\D(X, Y)$ with the best fit regression as in~\eqref{align:distance-regression}.
  }
  \label{fig:all-mu-values}
\end{figure}

\begin{figure}
  \centering 
  \includegraphics[width=0.9\textwidth]{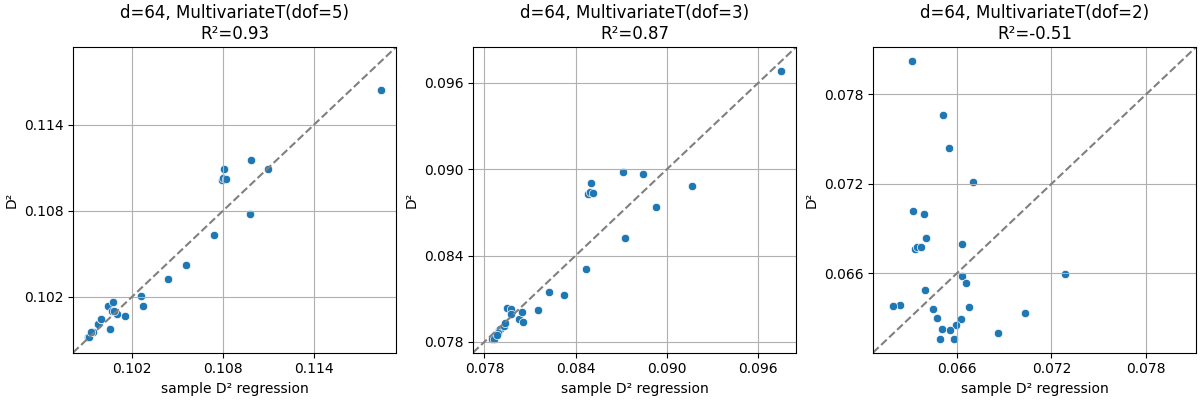}
  \includegraphics[width=0.9\textwidth]{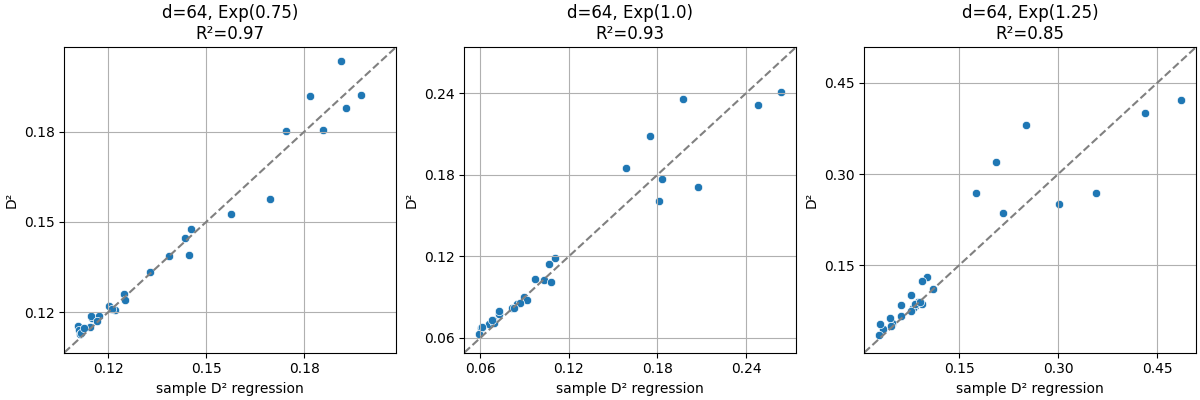}
  \includegraphics[width=0.9\textwidth]{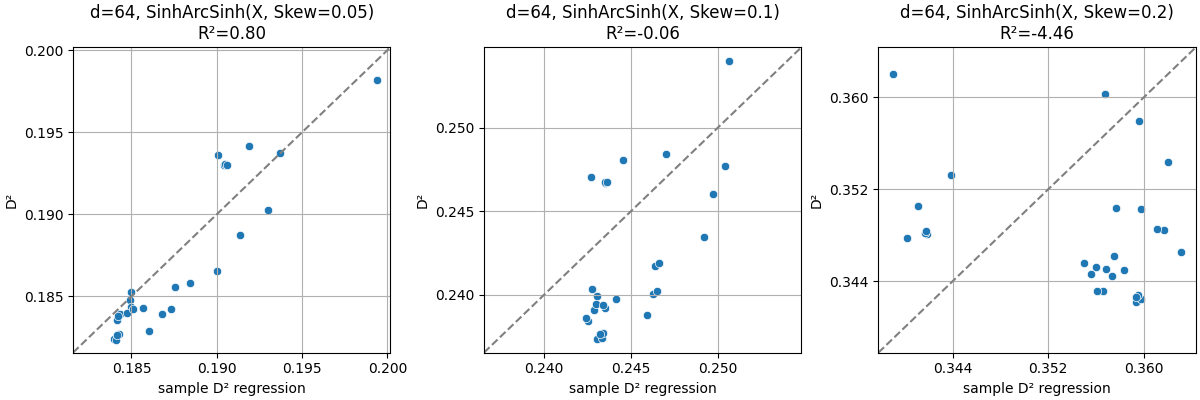}
  \caption{
      \textbf{Non-Gaussian distributions with larger perturbations, fixed $\mu_1$:} Here $X$ and $Y$ are transformed Gaussians. Only $d=64$ is shown. We only show $\mu_1=0.023$ so the effect of covariance can be isolated.
  We compare sample $\D(X, Y)$ with the best fit regression as in~\eqref{align:distance-regression}.
  }
  \label{fig:one-mu-values}
\end{figure}

\section*{Acknowledgments} The author would like to acknowledge conversations with Dmitrii Kochkov, Stephan Hoyer, and Janni Yuval, which motivated this work.

\bibliographystyle{plain}
\bibliography{energy-asymptotics}

\end{document}